\documentclass[11pt]{article}

\usepackage{amsfonts,amsmath,amssymb,amsthm,mathtools,bbm,bm}
\usepackage{physics,commath}
\usepackage{graphicx,caption,subcaption,hyperref,enumitem,authblk,multirow}
\usepackage[dvipsnames,svgnames]{xcolor}
\usepackage[nameinlink,capitalize,noabbrev]{cleveref}
\usepackage[numbers]{natbib}
\usepackage[toc,title]{appendix}
\usepackage{fancyhdr}
\usepackage{stickstootext} %fourier
\usepackage[stickstoo,vvarbb]{newtxmath}

\hypersetup{colorlinks=True, citecolor=DodgerBlue, linkcolor=MediumBlue, urlcolor=Blue}

\newtheorem{theorem}{Theorem}[section]

\newtheorem{proposition}[theorem]{Proposition}
\newtheorem{corollary}[theorem]{Corollary}
\theoremstyle{definition}
\newtheorem{definition}{Definition}

\newtheorem{assumption}{Assumption}
\newenvironment{assumption*}[1]{\renewcommand\theassumption{#1}\assumption}{\endassumption}

\numberwithin{equation}{section}
\newenvironment{equations}{\equation\aligned}{\endaligned\endequation}

\makeatletter
\newcommand*{\rom}[1]{\expandafter\@slowromancap\romannumeral #1@}
\makeatother

\def\iidsim{\stackrel{\textup{iid}}{\sim}}
\def\indsim{\stackrel{\textup{ind}}{\sim}}
\def\zero{\mathbf 0}

\def\ind{\mathbbm{1}}

\def\Ber{\texttt{\textup{Bernoulli}}}

\def\BETA{\texttt{\textup{Beta}}}

\def\IG{\texttt{\textup{IG}}}
\def\pois{\texttt{\textup{Poisson}}}

\def\N{\texttt{\textup{N}}}
\def\CIBP{\texttt{\textup{CIBP}}}
\def\IBP{\texttt{\textup{IBP}}}
\def\BP{\texttt{\textup{BP}}}
\def\BeP{\texttt{\textup{BeP}}}
\def\CRM{\texttt{\textup{CRM}}}

\def\d{\textup{d}}
\def\e{\textup{e}}

\def\E{\mathsf{E}}
\def\P{\mathsf{P}}
\def\R{\mathbb{R}}

\makeatletter
\def\greekvectors#1{%
 \@for\next:=#1\do{%
    \def\X##1;{\expandafter\def\csname b##1\endcsname{\bm{\csname##1\endcsname}}}
    \expandafter\X\next;}
 \@for\next:=#1\do{%
    \def\X##1;{\expandafter\def\csname h##1\endcsname{\widehat{\csname##1\endcsname}}}
    \expandafter\X\next;}
 \@for\next:=#1\do{%
    \def\X##1;{\expandafter\def\csname c##1\endcsname{\check{\csname##1\endcsname}}}
    \expandafter\X\next;}
 \@for\next:=#1\do{%
    \def\X##1;{\expandafter\def\csname hb##1\endcsname{\widehat{\bm{\csname##1\endcsname}}}}
    \expandafter\X\next;}
}
\greekvectors{alpha,beta,gamma,delta,epsilon,zeta,theta,kappa,lambda,mu,nu,xi,pi,rho,tau,phi,chi,psi,omega,
    Delta,Gamma,Theta,Lambda,Xi,Pi,Sigma,Phi,Psi,Omega}
    
\@tfor\next:=abcfghijlmnopqrstuvwxyzABCDGHIJLMOQSTUVWXYZ\do{%
    \def\command@factory#1{\expandafter\def\csname #1\endcsname{\mathbf{#1}} }
    \expandafter\command@factory\next
}
\@tfor\next:=abcdeghijklnopqrstuvwxyzABCDEFGHIJKLMNOPQRSTUVWXYZ\do{%
    \def\command@factory#1{\expandafter\def\csname b#1\endcsname{\mathbbm{#1}} }
    \expandafter\command@factory\next
}
\@tfor\next:=ABCDEFGHIJKLMNOPQRSTUVWXYZ\do{%
    \def\command@factory#1{\expandafter\def\csname c#1\endcsname{\mathcal{#1}} }
    \expandafter\command@factory\next
}
\@tfor\next:=abdefghijklmnopqrstuvwxyzABCDEFGHIJKLMNOPQRSTUVWXYZ\do{%
    \def\command@factory#1{\expandafter\def\csname s#1\endcsname{\mathsf{#1}} }
    \expandafter\command@factory\next
}
\@tfor\next:=abcdefghjklmnopqrstuvwxyzABCDEFGHIJKLMNOPQRSTUVWXYZ\do{
    \def\command@factory#1{\expandafter\def\csname f#1\endcsname{\mathfrak{#1}} }
    \expandafter\command@factory\next
}

\def\u{\mathbf{u}}
\def\CIBP{\texttt{\textup{CIBP}}}
\def\ssCIBP{\mathsf{SSCIBP}}

\def\Knew{K^{\textup{new}}}

\begin{document}
\title{{The convergent Indian buffet process}}
\author{Ilsang Ohn}
\affil{\textit{Department of Statistics, Inha University}}
\maketitle

\begin{abstract}
We propose a new Bayesian nonparametric prior for latent feature models, which we call the convergent Indian buffet process (CIBP). We show that under the CIBP, the number of latent features is distributed as a Poisson distribution with the mean monotonically increasing but converging to a certain value as the number of objects goes to infinity. That is, the expected number of features is bounded above even when the number of objects goes to infinity, unlike the standard Indian buffet process under which the expected number of features increases with the number of objects. We provide two alternative representations of the CIBP based on a hierarchical distribution and a completely random measure, respectively, which are of independent interest. The proposed CIBP is assessed on a high-dimensional sparse factor model.

\null\noindent
\end{abstract}

%%%%%%%%%%%%%%%%%%%%%%%%%%%%%%%%%%%%%%%%%%%%%%%%%%%%%%%%%%%%%%%%%%%%%%%%%%%%
\section{Introduction}

In this paper, we introduce a new three-parameter generalization of the Indian buffet process (IBP). The IBP, which is firstly introduced by \cite{griffiths2005infinite}, is an exchangeable distribution over binary matrices with a finite number of rows but an infinite number of columns. In the context of latent feature models, a binary matrix $\bXi:=(\xi_{jk})_{j\in\{1,\dots,p\}, k\in\bN}$ for $p\in\bN$ describes feature allocation for $p$ objects by letting $\xi_{jk}=1$  if the $j$-th object possesses the $k$-th feature and  $\xi_{jk}=0$ otherwise. The IBP and its two- and three-parameter generalizations \cite{thibaux2007hierarchical, teh2009indian} has been widely used in various applications \cite[e.g.,][]{meeds2007modeling, navarro2008latent, miller2009nonparametric, caron2012bayesian}.

It is well known that the expected number of features increases in a certain rate (logarithmic or polynomial) as the number of objects increases under both the one-, two-, and three-parameter IBPs \cite{griffiths2005infinite, thibaux2007hierarchical, teh2009indian}. Therefore, these IBPs, which can produce many unnecessary features, may not be suitable for modelling data sets that are believed to have a finite number of features. For example, in macroeconomic applications, fluctuations in data such as stock return can boil down to several important sources, so it is natural to assume that the number of features is fixed even if the data dimension increases \citep{onatski2010determining}.

In this paper, we propose a new stochastic process for latent feature models, under which the distribution of the number of features converges to a certain \textit{fixed}  distribution. Under the proposed process, fewer unnecessary features are generated than under the standard IBPs, and thus  both interpretability and prediction ability of the model can be improved.

\subsection{Convergent Indian buffet process}

Our proposed variant of the IBP, which we call the convergent Indian buffet process (CIBP) can be described by the following restaurant analogy. 

\begin{definition}[The restaurant analogy of the CIBP] Let $\gamma>0$, $\alpha>0$ and $\kappa\ge0$. We call the stochastic process  given below the restaurant analogy of $\CIBP(\gamma, \alpha, \kappa)$:
\begin{enumerate}
\item The first customer tries $\pois(\gamma\sB(\alpha+1,\kappa+1)/\sB(\alpha, \kappa+1))$ dishes, where $\sB(a, b)$ denotes the beta function with parameters $a$ and $b$.
\item For every $j = 2,\dots, p$, the $j$-th customer
	\begin{itemize}
	\item tries each previously tasted dish independently according to
	    \begin{equation}
	    \label{eq:prob_prev_dish}
	        \Ber\del{\frac{m_{j,k}+\alpha}{j+\kappa+\alpha}}
	    \end{equation}
	 where $m_{j,k}$ is the number of previous customers (before $j$-th customer) who have tried the $k$-th dish;
	\item and tries 
		  \begin{equation}
	    \label{eq:prob_new_dish}
	        \pois\del{\gamma \frac{\sB(\alpha+1,\kappa+j)}{\sB(\alpha, \kappa+1)}}
	    \end{equation}
	new dishes.
	\end{itemize}
\end{enumerate}
\label{def:restaurant}
\end{definition}

The restaurant analogy leads to the binary matrix with the number of rows being the number of costumers and the number of columns being unbounded, where the $(j,k)$-th element of the binary matrix is equal to 1 if the $j$-th customer tried the $k$-th dish and 0 otherwise. We denote by $\CIBP(\gamma, \alpha, \kappa)$ the distribution of the binary matrix induced by the above restaurant analogy. In this note, we discuss properties and construction of the CIBP.

\subsection{Organization}

The rest of the paper is organized as follows. In \cref{sec:properties}, we show that the number of features under the CIBP follows the Poisson distribution, with mean monotonically increasing but converging to a certain value as the number of objects $p$ goes to infinity. The name \textit{convergent} IBP is named after this property. We also describe connection between CIBP and the two-parameter IBP. In \cref{sec:construction}, we provides two alternative representations of the CIBP, where the first one is based on a hierarchical distribution of Poisson, Beta and Bernoulli distributions and the second one is based on random measures. In \cref{sec:factor}, as an application, we use the CIBP as the prior distribution on the factor loading matrix for Bayesian estimation of a sparse factor model. We provide a straightforward posterior computation algorithm and some numerical examples. In \cref{sec:proofs}, we give the proofs for the results of \cref{sec:construction}. \cref{sec:conclusion} concludes the paper.

\subsection{Notation} 
We denote by $\ind(\cdot)$ the indicator function. Let $\R$ be the set of real numbers and $\bR_+$ be the set of positive numbers.  Let $\bN$ be the set of natural numbers. For $m\in\bN$, we let $[m]:=\{1,\dots, m\}$. For noational convenience, we let $\bar{\sB}_{a_1,b_1}^{a_2, b_2}$ be the ratio of two beta functions defined as 
    \begin{equation*}
        \bar{\sB}_{a_1,b_1}^{a_2, b_2}:= \frac{\sB(a_1+a_2, b_1+b_2)}{\sB(a_1,b_1)}.
    \end{equation*}

\section{Properties}
\label{sec:properties}

\subsection{Distribution of the number of features}

In this section, we show that the number of features (i.e., dishes) under the CIBP follows a Poisson distribution with mean being fixed as the number of objects increases. Let $K^+$ be the number of nonzero columns of $\Xi$, which represents the number of features. Formally, we can define
    \begin{equation*}
        K^+:=K^+(\bXi):=\sum_{k=1}^\infty\ind(\bxi_{\bullet k}\neq\zero),%=\sum_{\u\in\Delta_1}K_\u
    \end{equation*}
where $\bxi_{\bullet k}$ denotes the $k$-th column of $\bXi$. The following proposition describes the distribution of $K^+$.

\begin{proposition}
If $\bXi\sim \CIBP(\gamma, \alpha, \kappa)$, then
    \begin{equation}
    \label{eq:dist_k}
        K^+\sim \pois\del{\gamma(1-\bar{\sB}_{\alpha, \kappa+1}^{0,p})},
    \end{equation}
where $\bar{\sB}_{\alpha, \kappa+1}^{0,p}:=\sB(\alpha,\kappa+p+1)/\sB(\alpha, \kappa+1)\le 1.$
Moreover, the Poisson mean $\gamma(1-\bar{\sB}_{\alpha, \kappa+1}^{0,p})$  monotonically increases and converges to $\gamma$ 
as $p\to\infty$, which, in particular, implies that $K^+$ converges to the random variable $K\sim \pois(\gamma)$ in distribution.
\end{proposition}

\begin{proof}
From the restaurant analogy of the CIBP, we have that
    \begin{equation*}
        K^+ \stackrel{d}{=} \sum_{j=1}^P\Knew_j, \mbox{ where }\Knew_j\indsim \pois\del{\gamma \frac{\sB(\alpha+1,\kappa+j)}{\sB(\alpha, \kappa+1)}}
    \end{equation*}
Therefore, by the additive property of independent Poisson random variables, 
    \begin{equation*}
        K^+ \sim \pois\del{\frac{\gamma}{\sB(\alpha, \kappa+1)}\sum_{j=1}^p\sB(\alpha+1,\kappa+j) }
    \end{equation*}
From the identity $\sB(x,y)-\sB(x, y+1)=\sB(x+1,y)$, we have
    \begin{align*}
        \sum_{j=1}^p\sB(\alpha+1,\kappa+j)
        &=\sum_{j=1}^p\cbr{\sB(\alpha, \kappa+j)-\sB(\alpha,\kappa+j+1)} \\
        &=\sB(\alpha, \kappa+1)-\sB(\alpha,\kappa+p+1),
    \end{align*}
which implies \labelcref{eq:dist_k}. The fact that  $\bar{\sB}_{\alpha, \kappa+1}^{0,p}\le 1$ follows from that $\sB(\alpha, \kappa+1)-\sB(\alpha,\kappa+p+1)\ge0$.

For the second assertion, note that
    \begin{align*}
        \bar{\sB}_{\alpha, \kappa+1}^{0,p}
        =\frac{\Gamma(\alpha)\Gamma(\kappa+p+1)}{\Gamma(\alpha+\kappa+p+1)}\frac{\Gamma(\alpha+\kappa+1)}{\Gamma(\alpha)\Gamma(\kappa+1)}
        =\prod_{j=1}^p\frac{\kappa+j}{\alpha+\kappa+j},
    \end{align*}
where $\Gamma(\bullet)$ denotes the gamma function. Since $\alpha>0$, it follows that $\bar{\sB}_{\alpha, \kappa+1}^{0,p}\downarrow 0$ as $p\to\infty$.
\end{proof}

\subsection{Exchangeability}

Exchangeability of the IBP makes corresponding posterior computation algorithms tractable. The CIBP is an exchangeable distribution also, as shown in the following corollary. This is a direct consequence of \cref{prop:joint} and \cref{prop:urn} which are presented in the next section.

\begin{corollary}
Assume that a $p\times \infty$-dimensional binary matrix $\bXi$ follows $\CIBP(\gamma, \alpha, \kappa)$. Then the random vectors $\bxi_{1\bullet},\dots, \bxi_{p\bullet}$ are exchangeable, where $\bxi_{j\bullet}$ denotes the $j$-th row of the matrix $\bXi$.
\end{corollary}

\subsection{Connection to the two-parameter IBP}

The  restaurant analogy  of the two-parameter IBP with parameters $\omega$ and $\kappa$ is as follows: The first customer tries $\pois(\omega)$ dishes. The $j$-th customer for $j\ge2$ tries each previously tasted dish independently according to $\mathsf{Bernoulli}(m_{j,k}/(j+\kappa))$ and tries $\pois(\omega\kappa/(j+\kappa))$ new dishes. We denote by  $\IBP(\omega, \kappa)$ the distribution induced by the above restaurant analogy.

By comparing the  restaurant analogies  of $\CIBP(\gamma, \alpha, \kappa)$ and $\IBP(\omega, \kappa)$, we then have the following proposition that connects these two stochastic processes.

\begin{proposition}
For two $p\times \infty$-dimensional binary matrices $\bXi\sim \CIBP(\gamma, \alpha, \kappa)$ and $\bXi_0\sim \IBP(\omega, \kappa)$, $\bXi$ converges to $\bXi_0$ in distribution as $\alpha\to0$ and $\gamma\alpha/\kappa \to \omega$. 
\end{proposition}

\begin{proof}
It suffices to show that the means of the Bernoulli distribution in \labelcref{eq:prob_prev_dish} and the Poisson distribution in  \labelcref{eq:prob_new_dish} converge to the corresponding quantities for $\IBP(\omega, \kappa)$, which can be derived as:
    \begin{equation*}
        \frac{m_{j,k}+\alpha}{j+\kappa+\alpha}\to  \frac{m_{j,k}}{j+\kappa}
    \end{equation*}
as $\alpha\to\zero$ and
    \begin{align*}
        \gamma \frac{\sB(\alpha+1,\kappa+j)}{\sB(\alpha, \kappa+1)}
         &=\gamma \frac{\Gamma(\alpha+1)\Gamma(\kappa+j)}{\Gamma(\alpha+\kappa+j+1)}\frac{\Gamma(\alpha+\kappa+1)}{\Gamma(\alpha)\Gamma(\kappa+1)} \\
         &= \gamma\frac{\alpha}{\alpha+\kappa+j} \prod_{h=1}^j\frac{\kappa+h}{\alpha+\kappa+h}\\
         &\to \frac{\omega \kappa }{j+\kappa},
    \end{align*}
as $\alpha\to0$  and $\gamma\alpha/\kappa \to \omega$.
\end{proof}

We visualize the result of the above propostion. \cref{fig:CIBP_ibp} shows four binary matrices generated by  $\CIBP(\omega\kappa/\alpha, \alpha, \kappa)$ and $\IBP(\omega,\kappa)$ with $\omega=5$, $\kappa=4$ but with $\alpha=5$, $\alpha=1$ and $\alpha=0.5$. We can see that the IBP tends to generate more features than the CIBP.

\begin{figure}
    \centering
    \includegraphics[scale=0.45]{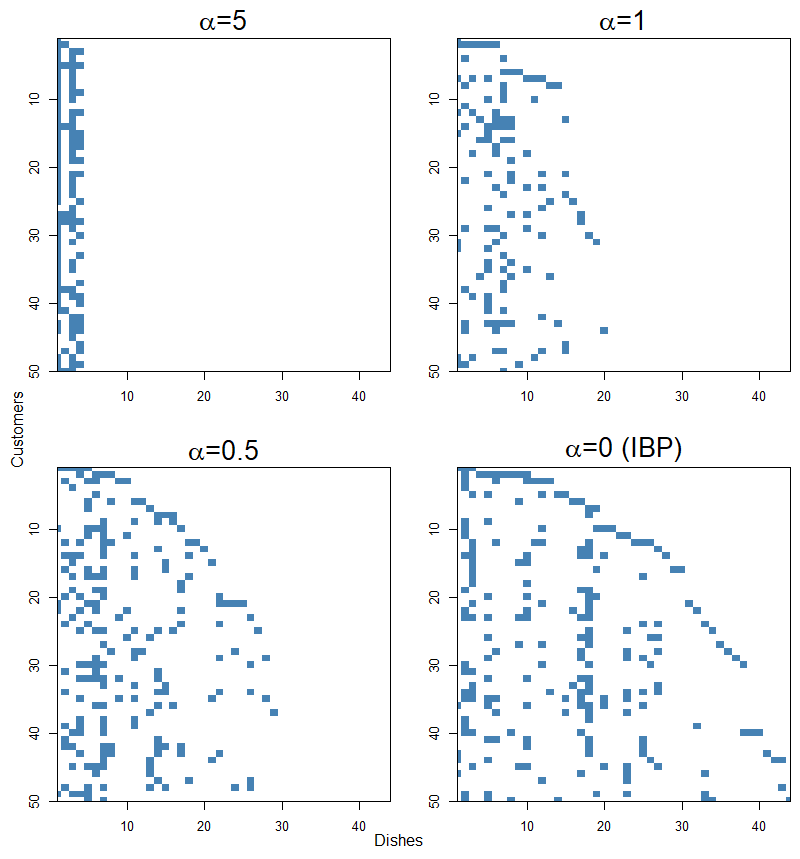}
    \caption{Draws from $\CIBP(\omega\kappa/\alpha, \alpha, \kappa)$ and $\IBP(\omega,\kappa)$ with $\omega=5$, $\kappa=4$ but with $\alpha=5$, $\alpha=1$ and $\alpha=0.5$.}
    \label{fig:CIBP_ibp}
\end{figure}

\section{Alternative representations}
\label{sec:construction}

In this section, we provides two alternative representations of the CIBP. The first one is based on a hierarchical distribution of Poisson, Beta and Bernoulli distributions and the second one is based on random measures. The proofs of all the results in this section are deferred to \cref{sec:proofs}.

\subsection{Hierarchical representation}

In this section we show that the CIBP is equivalent to the following hierarchical distribution.

\begin{definition}[Hierarchical representation of the CIBP] Let $\gamma>0$, $\alpha>0$ and $\kappa\ge0$. We call the  probability distribution given below the hierarchical representation of $\CIBP(\gamma, \alpha, \kappa)$:
\label{def:hierarchical}
	\begin{equations}
	\label{eq:hierarchical}
	K &\sim \pois(\gamma), \\
	\theta_k &\iidsim \BETA\del{\alpha,\kappa + 1},  \: k\in[K] \\
	\xi_{jk}|\theta_k &\indsim \Ber(\theta_k)\: j\in[p], k\in[K].
	\end{equations}

\end{definition}

To state the result rigorously, we need a concept of \textit{lof-equivalence} classes. Under the latent feature model, the ordering of the features does not affect the likelihood of the data. Hence, we say that two $p\times\infty$ dimensional binary matrices are equivalent if they are identical up to a permutation of columns. It is convenient to choose a representative  of every equivalence class by the \textit{left-ordering procedure}. The left-ordering procedure maps each $p\times\infty$ dimensional binary matrix to its left-ordered version whose columns are ordered by the score $s_k$, which is defined by
    \begin{equation*}
        s_k:=\sum_{j=1}^p\xi_{jk}2^{p-j}
    \end{equation*}
i.e., the colums are ordered so that $s_1\ge s_2\ge\cdots.$ We call the equivalence class defined by the left-ordering procedure \textit{lof-equivalence} class and we denote the lof-equivalence class of a binary matrix $\bXi$ by $[\bXi]$.

We introduce useful notations. Let $\Delta:=\{0,1\}^p$ which is a set of $p$-dimensional binary vectors and  $\Delta_{1}:=\Delta\setminus\{\zero\}$ where $\zero$ is the vector or zero.
For each $\u\in\Delta_1$, we define
    \begin{align}
        K_{\u} &:=\sum_{k=1}^\infty\ind(\bxi_{\bullet k}=\u),
    \end{align}
where $\bxi_{\bullet k}$ denotes the $k$-th column of $\bXi$. In words, $K_{\u}$ is the number of columns equal to the binary vector $\u$. Note that $K^+:=\sum_{k=1}^\infty\ind(\bxi_{\bullet k}\neq\zero)=\sum_{\u\in\Delta_1}K_\u$. Moreover, let
    \begin{equation*}
        m_k:=\sum_{j=1}^p\xi_{jk},
    \end{equation*}
be the number of rows that have the $k$-th feature. 

In the next proposition, we provide the explicit form of the probability mass function of the lof-equivalence class $[\bXi]$.

\begin{proposition}
\label{prop:joint}
If a $p\times \infty$-dimensional random binary  matrix $\Xi\equiv(\xi_{jk})_{j\in[p], k\in\bN}$ follows the distribution in  \labelcref{eq:hierarchical}, then
	\begin{equations}
	\P([\Xi]) =\frac{\gamma^{K^+}}{\prod_{\u\in\Delta_1}K_\u!}\e^{-\gamma \sum_{j=1}^p\bar{\sB}_{\alpha,\kappa+1}^{1,j-1}}\sbr{\prod_{k=1}^{K^+}\bar{\sB}_{\alpha,\kappa+1}^{m_{k},p-m_{k}} }.
	\label{eq:joint}
	\end{equations}
\end{proposition}

From \cref{prop:joint}, we can show that  the restaurant analogy and the hierarchical representation of the CIBP are equivalent.

\begin{proposition}
\label{prop:urn}
Suppose that a $p\times\infty$-dimensional binary matrix $\bXi$ follow the hierarchical distribution presented in \labelcref{eq:hierarchical}. Then the lof-equivalence class  $[\Xi]$ follows $\CIBP(\gamma, \alpha, \kappa)$.
\end{proposition}

\subsection{Random measure representation}

In this section, we provide another representation of the CIBP, which is based on random measures. 

We first briefly review completely random measures. Let $(\Omega, \cA)$ a Polish space with its Borel $\sigma$-field and let $(\fM, \cM)$ be a set of all measures on  $(\Omega, \cA)$ with its Borel $\sigma$-field. A completely random measure  (CRM) $\mu$ on $(\Omega, \cA)$ is a random measure such that $\mu(A_1),\dots, \mu(A_k)$ for all disjoint measurable sets $A_1,\dots, A_k\in \cA$ are mutually independent. Every CRM can be decomposed into three independent parts: 
    \begin{align*}
        \mu = \mu_0+\sum_{k=1}^Kq_k\delta_{\omega_k}+\sum_{( q,\omega )\in \Phi}q \delta_\omega   
    \end{align*}
where $\mu_0$ is a non-random measure, $(\omega_k)_{k\in[K]}$ are fixed atoms in $\Omega$, $(q_k)_{k\in[K]}$ are independent random variables on $\R_+$ and $\Phi$ is a Poisson process on $\R_+\times \Omega$. Here we only consider purely-atomic CRMs such that $\mu_0=0$. We write
$$\mu \sim \CRM\del{\Lambda, (\omega_k, P_k)_{k\in[K]}}$$
if $\mu$ is the purely-atomic CRM represented by $\mu = \sum_{k=1}^Kq_k\delta_{\omega_k}+\sum_{(q, \omega)\in \Phi}q\delta_\omega$ with $q_k\indsim P_k$ for $k\in[K]$ and $\E \Phi=\Lambda$ for some probability measures $(P_k)_{k\in [K]}$ on $\R_+$ and $\Lambda$ on $\R_+\times \Omega$. In particular, we write $\mu\sim \CRM\del{\Lambda}$ if $\mu=\sum_{(q, \omega)\in \Phi}q \delta_\omega$ with $\E \Phi=\Lambda$.

It is well known that the two-parameter IBP, $\IBP(\alpha, \kappa+1)$ with $\alpha>0$ and $\kappa\ge0$, has the following random measure representation:
	\begin{align*}
	\bxi_{j\bullet}|\mu &\iidsim\BeP(\mu), \:j\in[p]\\
	\mu &\sim	 \BP(\kappa+1, \alpha\Lambda_0)
	\end{align*}
for some smooth probability measure $\Lambda_0$, i.e., $\Lambda_0(\Omega)=1$. Here, $\BeP(\mu)$ denotes the Bernoulli process with mean $\mu$, which is equivalent to $\CRM(\Lambda_{\BeP(\mu)})$ on  $(\Omega, \cA)$ with
    \begin{align*}
        \Lambda_{\BeP(\mu)}(\d q, \d\omega)=\delta_1(\d q)\mu(\d\omega),
    \end{align*}
where $\delta_1$ denotes a point mass at 1, and $\BP(\kappa+1,\alpha\Lambda_0)$ denotes the Beta process which is equivalent to $\CRM(\Lambda_{\textup{BP}(\theta, \gamma\Lambda_0)})$ on  $(\Omega, \cA)$  with 
    \begin{align*}
        \Lambda_{\BP(\kappa+1,\alpha\Lambda_0)}(\d q, \d\omega)
        =\alpha(\kappa+1) q^{-1}(1-q)^{\kappa}\d q\Lambda_0(\d\omega).
    \end{align*}
    
We introduce another stochastic process represented by a random measure, which will be shown to be related to the CIBP.

\begin{definition}[Random measure representation of the CIBP]
\label{def:crm}
Let $\gamma>0$, $\alpha>0$ and $\kappa\ge0$. We call the  stochastic process given below the random measure representation of $\CIBP(\gamma, \alpha, \kappa)$: 
	\begin{equations}
	\label{eq:rm}
	\bxi_{j\bullet}|\mu &\iidsim\BeP(\mu), \:j\in[p]\\
	\mu &\sim\CRM(\Lambda_{\CIBP(\gamma, \alpha, \kappa)})
	\end{equations}
with
    \begin{equation}
        \Lambda_{\CIBP(\gamma, \alpha, \kappa)}(\d q, \d \omega)=\frac{\gamma}{\sB(\alpha,\kappa+1)} q^{\alpha-1}(1-q)^\kappa \d q\Lambda_0(\d\omega)
    \end{equation}
for some smooth  probability measure $\Lambda_0$. 
\end{definition}

The next theorem shows that the hierarchical representation in \cref{def:hierarchical} and random measure representation in \cref{def:crm} of the CIBP  are equivalent.

\begin{proposition} 
\label{prop:rm}
Let $\bxi_{1\bullet}, \dots, \bxi_{p\bullet}$ be random measures following the distribution given in \labelcref{eq:rm}. Then the joint distribution of $\bxi_{1\bullet}, \dots, \bxi_{p\bullet}$ is given by
	\begin{equations}
	\P(\bxi_{1\bullet}, \dots, \bxi_{p\bullet}) 
	=\e^{-\gamma \sum_{j=1}^p\bar{\sB}_{\alpha,\kappa+1}^{1,j-1}}\sbr{\prod_{k=1}^{K^+}\bar{\sB}_{\alpha,\kappa+1}^{m_{k},p-m_{k}} \lambda_0(\omega_k^*)},
	\end{equations}
where there are $K^+$ atoms $\omega_1^*, \cdots, \omega_{K^+}^*$ such that $m_k:=\sum_{j=1}^p\bxi_{j\bullet}(\omega_k^*)\ge1$ for  $k\in[K^+]$, and $\lambda_0$ denotes the density of $\Lambda_0$.
\end{proposition}

The function $q \mapsto q^{\alpha-1}(1-q)^\kappa$ is integrable on $[0,1]$, which means that there would be a finite number of features.

\section{Application to Bayesian sparse factor models}
\label{sec:factor}

In this section, we consider an application of the CIBP prior distribution to Bayesian estimation of the factor model.

\subsection{Model and prior}

We consider  the following factor model where  a $p$-dimensional random vector $\Y$ is distributed as
   	\begin{equations}
   	\label{eq:model_fac}
	\Y|\Z=\z \sim \N_p(\B\z, \sigma^2\I), \quad \Z\sim \N_K(\zero, \I),
	\end{equations}
with $K<p$, $\B$ being a $p\times K$ factor loading matrix, $\Z$  a $K$-dimensional factor and $\sigma^2>0$  a noise variance. 

We consider the following prior on the loading matrix $\B$. Let $\beta_{jk}$ be the $(j,k)$-th entry of the $p\times \infty$-dimensional loading matrix $\B$. We impose the prior distribution based on the CIBP distribution such that
    \begin{align*}
    \beta_{jk}|\xi_{jk}  \indsim &(1-\xi_{jk})\delta_0+\xi_{jk} \N(0, \tau),
    \: j\in[p], \: k\in[K] \\
    \xi_{jk}|\theta_k \indsim &\Ber\del{\theta_k},
    \: j\in[p], \: k\in[K] \\
    \theta_k \iidsim &\BETA(\alpha, \kappa+1),\:  k\in[K]\\
    K\sim &\pois(\gamma)
    \end{align*}
where $\kappa\ge 0$ and $\tau>0$. That is, we impose $\CIBP(\gamma,\alpha,\kappa)$ on the binary matrix $\bXi:=(\xi_{jk})_{j\in[p], k\in\bN}$. We refer to the above distribution on $\B$ as $\ssCIBP_{p}(\gamma, \alpha,\kappa,\tau)$, which is an abbreviation of {\it spike and slab CIBP}.

\subsection{Posterior computation}

We provide an Markov chain Monte Carlo (MCMC) algorithm for sampling from the posterior distribution under the $\ssCIBP_{p}(\gamma, \alpha,\kappa,\tau)$ prior on $\B$ and inverse Gamma prior $\IG(a,b)$ on $\sigma^2$. Let $K^+$ be the number of nonzero columns of the loading matrix $\B$. The MCMC algorithm is as follows:

\paragraph{Sample  $\beta_{jk}$ for $j\in[p]$ and $k\in[K^+]$.}
The factor loading $\beta_{jk}$ is sampled from the conditional posterior
    \begin{equation*}
        \beta_{jk}|- \sim
        \begin{cases}
        \N(\hbeta_{jk}, \htau_{k}) & \mbox{if }\xi_{jk}=1\\
        \delta_0 & \mbox{if }\xi_{jk}=0,
        \end{cases}
    \end{equation*}
where
    \begin{align*}
    \htau_{k}&:=\del{\sigma^{-2}\sum_{i=1}^nZ_{ik}^2+\tau^{-1}}^{-1}\\
    \hbeta_{jk}&:=\htau_{k}\cbr{\sigma^{-2}\sum_{i=1}^nZ_{ik}\del{Y_{ij}-\sum_{h\in[K^+]:h\neq k}Z_{ih}\beta_{jh}}}.
    \end{align*}

\paragraph{Sample  $\xi_{jk}$ for $j\in[p]$ and $k\in\bN$.}
When we sample $(\xi_{jk}:k\in\bN)$, we use the fact that the CIBP is exchangeable to assume that the $j$-th customer is the last customer to enter the restaurant. Therefore, for each $k\in[K^*]$, $\xi_{jk}$ is sampled with probability
    \begin{equation*}
        \frac{\Pi(\xi_{jk}=1|-)}{\Pi(\xi_{jk}=0|-)}
        = \frac{m_{j,k}+\alpha}{\kappa+p-m_{j,k}}\sqrt{\frac{\htau_{k}}{\tau}}\exp\del{\frac{1}{2\htau_{k}}\hbeta_{jk}^2},
    \end{equation*}
where $m_{j,k}:=\sum_{l\in[p]:l\neq j}\xi_{lk}$. We then sample $\xi_{jk}$ for each of the infinitely many all-zero columns. To do this, we use the Metropolis–Hastings (MH) steps as follows. We propose $K_j^*\in\bN\cup\{0\}$ and $\bbeta_j^{*}:=(\beta_{j,1}^{*},\dots, \beta_{j,K_j^*}^{*})\in\R^{K_j^{*}}$ from the proposal distribution 
    \begin{equation*}
        J(K_j^*)J(\beta_j^{*}|K_j^*)=\pois(1)\N(0,\tau)^{K_j^*}.%\sB(2, \kappa+p)/\sB(1,\kappa+1) \gamma\bar{\sB}^{1,p-1}_{\alpha,\kappa+1}
    \end{equation*}
%where the Poisson part of the right-hand side is equal to the prior distribution of the number of new dishes taken by the $p$-th customer and the normal part is equal to the prior distribution of the $K_j^*$ nonzero loadings. 
Then we accept the proposal with probability
    \begin{equation*}
        \min\cbr{1,\abs[0]{\M_j}^{-n/2}\exp\del{\frac{1}{2}(\bbeta_j^{*})^\top\M_j^{-1}\bbeta_j^{*}\sum_{i=1}^nE_{ij}^2}\del{\gamma\bar{\sB}^{1,p-1}_{\alpha,\kappa+1}}^{K_j^*}},
    \end{equation*}
where
    \begin{align*}
        \M_j&:=\sigma^{-2}\bbeta_j^{*}(\bbeta_j^{*})^\top+\I\\
        E_{ij}&:=\sigma^{-2}\del{Y_{ij}-\sum_{k=1}^{K^+}Z_{ik}\beta_{jk}}.
    \end{align*}
If the proposal is accepted, we update 
    \begin{align*}
        \B&\leftarrow (\B, (\beta_{j,k}^{*}\ind(l=j))_{l\in[p], k\in[K_j^*]})\\
        K^+&\leftarrow K^++K_j^*.
    \end{align*}

\paragraph{Sample $\Z_i$ for $i\in[n]$.} The latent variable $\Z_i$ is sampled from
    \begin{equation*}
        \Z_i|-\sim \N\del{ \sigma^{-2}\hbSigma_{\Z}\B^\top\Y_i, \hbSigma_{\Z}}
    \end{equation*}
where $\hbSigma_{\Z}:=( \sigma^{-2}\B^\top\B+\I)^{-1}$.
 
\paragraph{Sample $\sigma^2$.} The noise variance $\sigma^2$ is sampled from
    \begin{equation*}
         \sigma^2|-\sim \IG\del{a+ \frac{np}{2}, b+\frac{1}{2}\sum_{i=1}^n\sum_{j=1}^p\del{Y_{ij}-\sum_{k=1}^{K^+}Z_{ik}\beta_{jk}}^2}.
    \end{equation*}

\subsection{Simulation}

We conduct simulation to compare the CIBP and the two-parameter IBP  when they are used as prior distributions for the sparse factor model.

We generate simulated data sets as follows. For each value $p\in\{50, 100, 150,\dots, 300\}$, we generate a $p\times 4$-dimensional loading matrix $\B_0$ with the number of nonzero rows $10$. The loadings in the sampled nonzero rows are generated from the uniform distribution on $(-3,-2)\cup(2,3)$. Then we sample $n=50$ random vectors from the multivariate normal distribution with mean $\zero$ and variance $\B_0(\B_0)^\top+\I$ independently. We repeat this generating procedure 100 times.

For each synthetic data set, we compute the posterior distribution under the  $\CIBP(\gamma,\alpha,\kappa)$ and $\IBP(\omega,\kappa)$ prior, respectively. For the CIBP prior, we set $\gamma=1$, $\alpha=10$ and $\kappa=10$. For the IBP prior, we set $\omega=1$ which is equal to $\gamma\alpha/\kappa$. In \cref{fig:simul}, we present the posterior mean of the number of factors under the $\CIBP(\gamma,\alpha,\kappa)$ and $\IBP(\omega,\kappa)$ prior, respectively, for $p\in\{50, 100, 150,\dots, 300\}$ over 100 replications.  As the dimension $p$  increases, the $\IBP(\omega,\kappa)$ prior tends to more largely overestimate the number of factors. But the $\CIBP(\gamma,\alpha,\kappa)$ prior provides accurate estimates of the number of factors for all the values of $p.$

\begin{figure}
    \centering
    \includegraphics[scale=0.2]{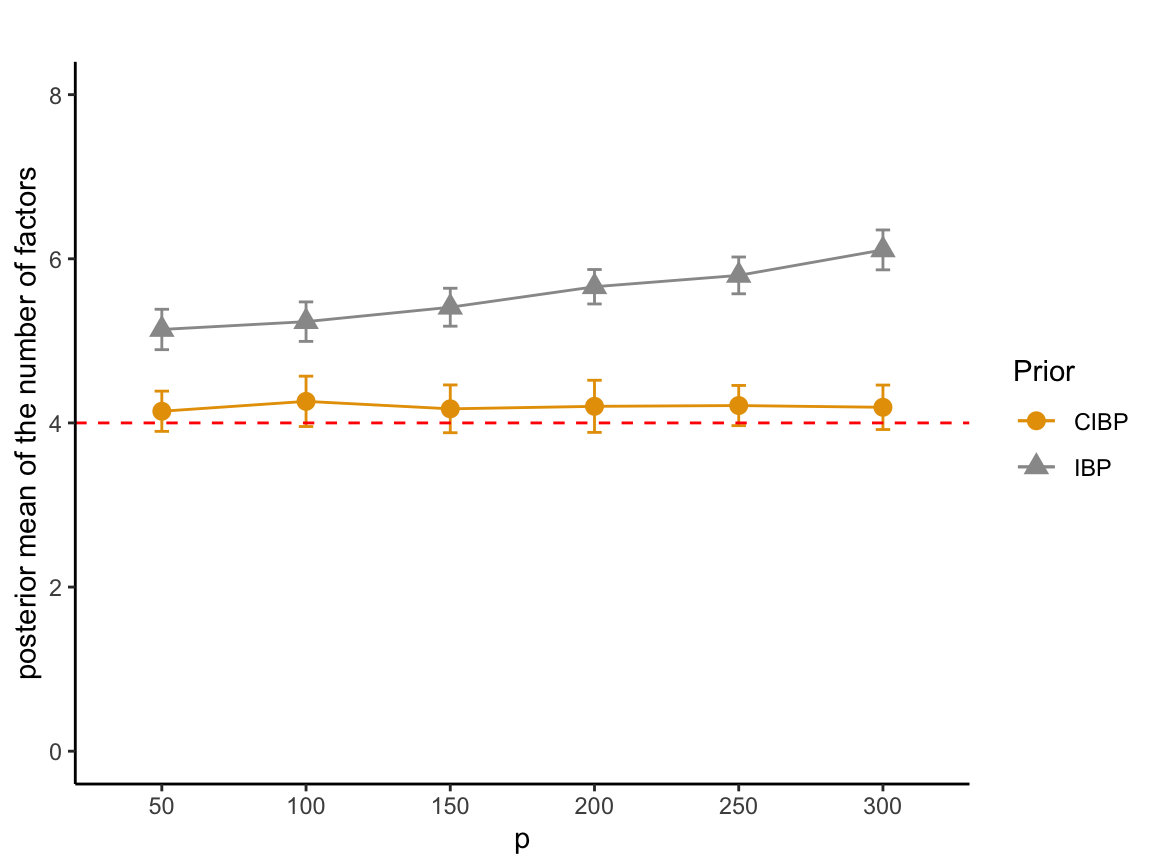}
    \caption{The posterior mean of the number of factors under the $\CIBP(\gamma,\alpha,\kappa)$ and $\IBP(\omega,\kappa)$ priors. The red dashed line indicates the true number of factors, 4.}
    \label{fig:simul}
\end{figure}

\section{Proofs for \cref{sec:construction}}
\label{sec:proofs}
%\subsection{Proofs of results in \cref{sec:equiv_rep}}
%\label{sec:CIBP_prf_rep}

\subsection{Proof of \cref{prop:joint}}

\begin{proof}
Recall that $m_{k}:=\sum_{j=1}^p\xi_{jk}$. If $K\ge K^+$,  we have that
	\begin{align*}
	\P(\bXi|K) &= \prod_{k=1}^K\frac{ \sB(m_{k}+\alpha, p-m_{k}+\kappa+1)}{\sB(\alpha, \kappa+1)}\\
	& =\del{\frac{\sB(\alpha, p+\kappa+1)}{\sB(\alpha, \kappa+1)}}^{K-K^+}
	\prod_{k=1}^{K^+}\frac{\sB(m_{k}+\alpha, p-m_{k}+\kappa+1)}{\sB(\alpha, \kappa+1)} \\
	& =\del{\bar\sB_{\alpha, \kappa+1}^{0,p}}^{K-K^+}\prod_{k=1}^{K^+}\bar\sB_{\alpha,\kappa+1}^{m_{k},p-m_{k}},
	\end{align*}
where the second equality follows from reordering the columns such that $m_{k}>0$ if $k\le K^+$ and $m_{k}=0$ otherwise. Recall that $\bar{\sB}_{a_1,b_1}^{a_2,b_2}=\sB(a_1+a_2, b_1+b_2)/\sB(a_1,b_1)$. Therefore, since the cardinality of the lof-equivalence class is $|[\bXi]]=K!/\prod_{\u\in \Delta}K_\u!$, the probability of a lof equivalence class of $\bXi$ given $K\ge K^+$ is given by
    \begin{equation*}
        \P([\bXi]|K) = \frac{K!}{\prod_{\u\in \Delta}K_\u!}\del{\bar\sB_{\alpha, \kappa+1}^{0,p}}^{K-K^+}
        \prod_{k=1}^{K^+}\bar\sB_{\alpha,\kappa+1}^{m_{k},p-m_{k}}
    \end{equation*}
If $K< K^+$, it is clear that $\P(\bXi|K)=0$. 

Let $p_K$ the probability mass function of $\pois(\gamma)$, i.e.,  $p_K(k):=\e^{-\gamma}\gamma^k/k!$ for $k\in\bN$. Marginalizing over $K$, we have that
	\begin{align*}
	\P([\bXi]) =\frac{1}{\prod_{\u\in \Delta_1}K_\u!}\sbr{\prod_{k=1}^{K^+}\bar\sB_{\alpha,\kappa+1}^{m_{k},p-m_{k}}}
\sum_{K= K^+}^{\infty}\frac{K!}{K_\zero!}\del{\bar{\sB}_{\alpha, \kappa+1}^{0,p}}^{K-K^+}p_K(K).
	\end{align*}
The summation term of the preceding display can be written as
	\begin{equations}
	\sum_{K= K_+}^{\infty}\frac{K!}{K_\zero!}\del{\bar{\sB}_{\alpha, \kappa+1}^{0,p}}^{K-K^+}p_K(k)
    &=\e^{-\gamma}\gamma^{K^+}\sum_{K= K_+}^{\infty}\frac{1}{(K-K^+)!}\del{\gamma \bar{\sB}_{\alpha, \kappa+1}^{0,p}}^{K-K^+}  \\
    &=\gamma ^{K^+}\e^{-\gamma\del{1-\bar{\sB}_{\alpha, \kappa+1}^{0,p}}},
	\label{eq:large}
	\end{equations}
where we use the identity $\e^x=\sum_{k=0}^\infty x^k/k!$ for the second inequality. 
Lastly, from the identity $\sB(x,y)-\sB(x, y+1)=\sB(x+1,y)$, it follows that  
	\begin{equations}
	1-\bar{\sB}_{\alpha, \kappa+1}^{0,p}
	&=1-\frac{\sB(\alpha,p+\kappa+1)}{\sB(\alpha, \kappa+1)} \\
	&=  \frac{1}{\sB(\alpha, \kappa+1)} \cbr{\sB(\alpha, \kappa+1)-\sB(\alpha,p+\kappa+1)}  \\
	&= \frac{1}{\sB(\alpha, \kappa+1)}\sum_{j=1}^p\cbr{\sB(\alpha, \kappa+j)-\sB(\alpha,\kappa+j+1)} \\
	&= \frac{1}{\sB(\alpha, \kappa+1)}\sum_{j=1}^p\sB(\alpha+1,\kappa+j) \\
	&=\sum_{j=1}^p\bar\sB_{\alpha, \kappa+1}^{1, j-1}.
\label{eq:beta}
	\end{equations}
Combining \labelcref{eq:large} and \labelcref{eq:beta}, we get the desired result.
\end{proof}

\subsection{Proof of \cref{prop:urn}}

\begin{proof}
The proof is by induction. Let $\bxi_{j\bullet}$ be the $j$-th row of $\bXi$.  For $p=1$, from a Poisson likelihood, we have 
    \begin{equation*}
        \P(\bxi_{1\bullet}) = \frac{1}{K_1^+!}\del{\gamma\bar{\sB}_{\alpha, \kappa+1}^{1,0}}^{K_1^+}\e^{-\gamma\bar{\sB}_{\alpha, \kappa+1}^{1,0}}
    \end{equation*}
where $K_1^+$ is a number of nonzero elements of $\bxi_{1\bullet}$. It is same as \labelcref{eq:joint} with $p=1$ and $K^+=K_1^+$. 

For $p\ge 2$, consider the conditional distribution of $\bxi_{p\bullet}$ given $\bxi_{1\bullet},\dots, \bxi_{p-1\bullet}$, which is given by
    \begin{equations}
    \label{eq:conditional}
        \P(\bxi_{p\bullet}|\bxi_{1\bullet},\dots, \bxi_{p-1\bullet})
        &=\e^{-\gamma\bar{\sB}_{\alpha, \kappa+1}^{1,p-1}}
        \frac{(\gamma\bar{\sB}_{\alpha,\kappa+1}^{1,p-1})^{\Knew_p}}{\Knew_p!}\\
        &\quad \times \prod_{k\in J_{p}}\frac{m_{p,k}+\alpha}{p+\kappa + \alpha}\prod_{k\notin J_{p}}\frac{p-m_{p,k}+\kappa}{p+\kappa+\alpha},
    \end{equations}
where $m_{p,k}:=\sum_{j=1}^{p-1}\xi_{jk}$, $\Knew_p$ is the number of new features sampled by the $p$-th customer and $J_p$ is the set of dishes taken by the $p$-th customer, i.e., $J_p:=\cbr{k\in[K_{p-1}^+]:\xi_{pk}=1}$. Let $K_p^+:=\sum_{j=1}^p\Knew_j=K_{p-1}^++\Knew_p$ and $\Knew_1=K_1^+$.
By the inductive hypothesis, we have
    \begin{align*}
    \P(\bxi_{1\bullet},\dots, \bxi_{p\bullet}) 
    & =\P(\bxi_{p\bullet}|\bxi_{1\bullet},\dots, \bxi_{p-1\bullet})\P(\bxi_{1\bullet},\dots, \bxi_{p-1\bullet})\\
    & =\e^{-\gamma\sum_{j=1}^p \bar{\sB}_{\alpha,\kappa+1}^{0,j-1}}\frac{\gamma^{K_p^+}}{\prod_{j=1}^p \Knew_j!}
    \prod_{k\in J_{p}}\frac{m_{p,k}+\alpha}{p+\kappa+\alpha}\bar{\sB}_{\alpha,\kappa+1}^{m_{p,k},p-1-m_{p,k}}  \\
    &\quad \times \prod_{k\notin J_p}\frac{p-m_{p,k}}{p+\kappa+\alpha}\bar{\sB}_{\alpha,\kappa+1}^{m_{p,k},p-1-m_{p,k}}  \times \del{\bar{\sB}_{\alpha, \kappa+1}^{1,p-1}}^{\Knew_p}
\end{align*}
Since $m_k=m_{p,k}+1$ for $k\in J_p$ and $m_{k}=m_{p,k}$ otherwise, we have 
    \begin{align*}
        \frac{m_{p,k}+\alpha}{p+\kappa+\alpha}\bar{\sB}_{\alpha,\kappa+1}^{m_{p,k},p-1-m_{p,k}}
        &=\frac{m_{p,k}+\alpha}{p+\kappa+\alpha}\frac{\sB(m_{p,k}+\alpha, p -m_{p,k}+\kappa)}{\sB(\alpha,\kappa+1)}\\
        &	=\frac{\sB(m_{p,k}+1+\alpha, p-m_{p,k}+\kappa)}{\sB(\alpha,\kappa+1)}\\
        &	=\bar{\sB}_{\alpha,\kappa+1}^{m_{p,k}+1,p-1-m_{p,k}}\\
        &	=\bar{\sB}_{\alpha,\kappa+1}^{m_{k},p-m_{k}}
    \end{align*}
and similarly, 
    \begin{align*}
        \frac{p-m_{p,k}}{p+\kappa+\alpha}\bar{\sB}_{\alpha,\kappa+1}^{m_{p,k},p-1-m_{p,k}}
        =\bar{\sB}_{\alpha, \kappa+1}^{m_{p,k},p-m_{p,k}}=\bar{\sB}_{\alpha, \kappa+1}^{m_{k},p-m_{k}}.
    \end{align*}
Therefore, 
    \begin{equations}
    \P(\bxi_{1\bullet},\dots, \bxi_{p\bullet})
    &= \e^{-\gamma\sum_{j=1}^p \bar{\sB}_{\alpha,\kappa+1}^{1,j-1}}\frac{\gamma^{K_p^+}}{\prod_{j=1}^p \Knew_j!}
	    \prod_{k\in J_{p}}\bar{\sB}_{\alpha,\kappa+1}^{m_{k},p-m_{k}}\\
	 &\quad \times  \prod_{k\notin J_p} \bar{\sB}_{\alpha, \kappa+1}^{m_{k},p-m_{k}}
	   \times\del{\bar{\sB}_{\alpha, \kappa+1}^{1,p-1}}^{\Knew_p} \\
   & =\e^{-\gamma\sum_{j=1}^p \bar{\sB}_{\alpha,\kappa+1}^{1,j-1}}\frac{\gamma^{K_p^+}}{\prod_{j=1}^p \Knew_j!}
	    \prod_{k=1}^{K_p^+}\bar{\sB}_{\alpha, \kappa+1}^{m_{k},p-m_{k}}. 
    \label{eq:induction}
    \end{equations}
Note that $\prod_{j=1}^p \Knew_j!/\prod_{\u\in\Delta_1} K_\u$ matrices generated by the above process have the same left-ordered form, hence $\P([\bXi])$ is obtained by multiplying $ \P(\bxi_{1\bullet},\dots, \bxi_{p\bullet})$ in \labelcref{eq:induction} by this quantity.
\end{proof}

\subsection{Proof of \cref{prop:rm}}

\begin{proof}
By the well-known conjugacy result (Theorem 3.3 of \citet{kim1999nonparametric}), 
    \begin{equation*}
        \mu|\bxi_{1\bullet},\dots, \bxi_{p-1\bullet} \sim \CRM(\Lambda_p, \{\omega_k^*,P_{k}\}_{k=1}^K),
    \end{equation*}
where $\omega_1^*,\dots, \omega_K^*$ are unique atoms that $\bxi_{1\bullet},\dots, \bxi_{p-1\bullet}$ possess,
	\begin{align*}
	P_k(\d q)&:=\frac{q^{m_{p,k}+\alpha-1}(1-q)^{p-1-m_{p,k}+\kappa}\d q}{\int_{(0,1]}q^{m_{p,k}+\alpha-1}(1-q)^{p-1-m_{p,k}+\kappa}\d q} \\
	&= \frac{1}{\sB(m_{p,k}+\alpha,p-m_{p,k}+\kappa)}q^{m_{p,k}+\alpha-1}(1-q)^{p-1-m_{p,k}+\kappa}\d q,
	\end{align*}
with $m_{p,k}:=\sum_{j=1}^{p-1}\bxi_{j\bullet}(\omega_k^*)$, and
    \begin{equation*}
        \Lambda_p(\d q, \d \omega):=\frac{\gamma}{\sB(\alpha,\kappa+1)} q^{\alpha-1}(1-q)^{p-1+\kappa} \d q \Lambda_0(\d\omega).
    \end{equation*}
Thus, for each atom $\omega_k^*$, we have that
	\begin{equations}
	\label{eq:crm_prev_dish}
	&\P\del{\bxi_{p\bullet}(\omega_k^*)=1|\bxi_{1\bullet},\dots, \bxi_{p-1\bullet}} \\
	&= \frac{1}{\sB(m_{p,k}+\alpha,p-m_{p,k}+\kappa)}\int_{(0,1]}q^{m_{p,k}+\alpha}(1-q)^{p-1-m_{p,k}+\kappa}\d q\\
	&= \frac{1}{\sB(m_{p-1,k}+\alpha,p-m_{p-1,k}+\kappa)}{\sB(m_{p,k}+1+\alpha, p-m_{p,k}+\kappa)} \\
	&= \frac{m_{p,k}+\alpha}{p+\kappa+\alpha}.
	\end{equations}
On the other hand, for a small neighborhood $\d \omega$ around $\omega\in \Omega\setminus\{\omega_1^*, \dots,\omega_K^*\}$, we have 
	\begin{equations}
	\label{eq:crm_new_dish}
	&\P\del{\bxi_{p\bullet}(\d\omega)=1|\bxi_{1\bullet},\dots, \bxi_{p-1\bullet}}\\
	&=\E\sbr{\bxi_{p\bullet}(\d\omega)|\bxi_{1\bullet},\dots, \bxi_{p-1\bullet}}	\\
	&=\frac{\gamma}{\sB(\alpha,\kappa+1)}\int_{(0,1]}q^\alpha(1-q)^{p-1+\kappa} \d q \Lambda_0(\d\omega) \\
	&=\gamma\frac{ \sB(\alpha+1, p+\kappa)}{\sB(\alpha,\kappa+1)} \Lambda_0(\d\omega) \\
	&=\gamma\bar\sB_{\alpha,\kappa+1}^{1,p-1}\Lambda_0(\d\omega)
	\end{equations}
This implies that on $\Omega\setminus\{\omega_1^*,\dots, \omega_K^*\}$, $\bxi_{p\bullet}$ is a Poisson process with intensity measure $\gamma\bar\sB_{\alpha,\kappa+1}^{1,p-1}\Lambda_0$, since $\bxi_{p\bullet}$ is completely random and $\Lambda_0$ is smooth. Thus, the number of new atoms in $\bxi_{p\bullet}$ follows Poisson distribution with rate  $\gamma\bar\sB_{\alpha,\kappa+1}^{1,p-1}$. Combining \labelcref{eq:crm_prev_dish} and \labelcref{eq:crm_new_dish}, we complete the proof.
\end{proof}

\section{Conclusion}
\label{sec:conclusion}

In this paper, we proposed the CIBP, a new  three-parameter generalization of the IBP. The most notable property of the CIBP compared with the standard IBPs is that the distribution of the number of features converges to a fixed distribution as the number of objects grows. Due to this property, the  CIBP prevents features that are unnecessary in interpretation and/or prediction.  The proposed CIBP is assessed on  a high-dimensional sparse factor model. We provided empirical evidence showing that the CIBP is better than the IBP in estimation of the number of factors.

\subsection*{Acknowledgement}
This work was supported by INHA UNIVERSITY Research Grant.

\bibliographystyle{apalike}
\bibliography{_references}

\begin{thebibliography}{}

\bibitem[Caron, 2012]{caron2012bayesian}
Caron, F. (2012).
\newblock Bayesian nonparametric models for bipartite graphs.
\newblock In {\em Advances in Neural Information Processing Systems}, pages
  2051--2059.

\bibitem[Griffiths and Ghahramani, 2005]{griffiths2005infinite}
Griffiths, T.~L. and Ghahramani, Z. (2005).
\newblock Infinite latent feature models and the indian buffet process.
\newblock In {\em Proceedings of the 18th International Conference on Neural
  Information Processing Systems}, pages 475--482.

\bibitem[Kim et~al., 1999]{kim1999nonparametric}
Kim, Y. et~al. (1999).
\newblock Nonparametric bayesian estimators for counting processes.
\newblock {\em The Annals of Statistics}, 27(2):562--588.

\bibitem[Meeds et~al., 2007]{meeds2007modeling}
Meeds, E., Ghahramani, Z., Neal, R.~M., and Roweis, S.~T. (2007).
\newblock Modeling dyadic data with binary latent factors.
\newblock In {\em Advances in Neural Information Processing Systems}, pages
  977--984.

\bibitem[Miller et~al., 2009]{miller2009nonparametric}
Miller, K., Jordan, M.~I., and Griffiths, T.~L. (2009).
\newblock Nonparametric latent feature models for link prediction.
\newblock In {\em Advances in Neural Information Processing Systems}, pages
  1276--1284.

\bibitem[Navarro and Griffiths, 2008]{navarro2008latent}
Navarro, D.~J. and Griffiths, T.~L. (2008).
\newblock Latent features in similarity judgments: {A} nonparametric bayesian
  approach.
\newblock {\em Neural computation}, 20(11):2597--2628.

\bibitem[Onatski, 2010]{onatski2010determining}
Onatski, A. (2010).
\newblock Determining the number of factors from empirical distribution of
  eigenvalues.
\newblock {\em The Review of Economics and Statistics}, 92(4):1004--1016.

\bibitem[Teh and Gorur, 2009]{teh2009indian}
Teh, Y.~W. and Gorur, D. (2009).
\newblock Indian buffet processes with power-law behavior.
\newblock In {\em Advances in Neural Information Processing Systems}, pages
  1838--1846.

\bibitem[Thibaux and Jordan, 2007]{thibaux2007hierarchical}
Thibaux, R. and Jordan, M.~I. (2007).
\newblock Hierarchical beta processes and the indian buffet process.
\newblock In {\em Artificial Intelligence and Statistics}, pages 564--571.

\end{thebibliography}
\end{document}